\documentclass{article}

\usepackage{microtype}
\usepackage{graphicx}
\usepackage{subfigure}
\usepackage{booktabs} %
\usepackage{framed}

\usepackage{hyperref}

\usepackage[accepted]{icml2019}

\icmltitlerunning{White-box vs Black-box: Bayes Optimal Strategies for Membership Inference}

\usepackage{url}

\usepackage{epsfig}
\usepackage{graphicx}
\usepackage{amsmath}
\usepackage{amssymb}
\usepackage{amsthm}
\usepackage{enumitem}
\usepackage{bbding}

\usepackage[latin1]{inputenc}

\usepackage{xspace}

\usepackage{booktabs}

\usepackage{color}
\usepackage{array,multirow}
\usepackage{listings}
\usepackage[title]{appendix}

\def \ie {\emph{i.e.}\xspace}
\def \eg {\emph{e.g.}\xspace}

\def \threshold {MAT\xspace}

	\newcommand {\mypar}[1]{\paragraph{#1}}

\definecolor{darkgreen}{RGB}{0, 140, 0}

\newcommand{\matthijs}[1]{{\color{blue}[\textbf{Matthijs}:#1]}}

\newcommand \prob {\mathbb{P}}

\newcommand \dist {\mathcal D}
\newcommand \iid {\underset{\text{i.i.d.}}{\sim}}
\newcommand \N {\mathbb N}

\newcommand \esp {\mathbb{E}}
\newcommand{\ZM}{\mathcal{T}}
\renewcommand{\det}[1]{\mathrm{det} \left( #1 \right) }

\def \tiny        {\textit{Tiny}\xspace}

\newtheorem{theorem}{Theorem}
\newtheorem{definition}{Definition}

\newtheorem{property}{Property}

\DeclareMathOperator{\Tr}{Tr}

\begin{document}

\twocolumn[
\icmltitle{White-box vs Black-box: Bayes Optimal Strategies for Membership Inference}

\icmlsetsymbol{equal}{*}

\begin{icmlauthorlist}
\icmlauthor{Alexandre Sablayrolles}{inria,fb}
\icmlauthor{Matthijs Douze}{fb}
\icmlauthor{Yann Ollivier}{fb}
\icmlauthor{Cordelia Schmid}{inria}
\icmlauthor{Herv\'e J\'egou}{fb}
\end{icmlauthorlist}

\icmlaffiliation{inria}{University Grenoble Alpes, Inria, CNRS, Grenoble INP, LJK}
\icmlaffiliation{fb}{Facebook AI Research}

\icmlcorrespondingauthor{Alexandre Sablayrolles}{asablayrolles@fb.com}

\icmlkeywords{Machine Learning}

\vspace{0.7cm}
]

\printAffiliationsAndNotice{}  %

\begin{abstract}
Membership inference determines, given a sample and trained parameters of a machine learning model, whether the sample was part of the training set.
In this paper, we derive the optimal strategy for membership inference with a few assumptions on the distribution of the parameters. 
We show that optimal attacks only depend on the loss function, and thus black-box attacks are as good as white-box attacks. 
As the optimal strategy is not tractable, we provide approximations of it leading to several inference methods, 
and show that existing membership inference methods are coarser approximations of this optimal strategy.
Our membership attacks outperform the state of the art in various settings, ranging from a simple logistic regression to more complex architectures and datasets, such as ResNet-101 and Imagenet. 

\end{abstract}

\newcommand{\dummyfig}[1]{
  \centering
  \fbox{
    \begin{minipage}[c][0.25\textheight][c]{\linewidth}
      \centering{#1}
    \end{minipage}
  }
}

\section{Introduction}

\citet{ateniese2015hacking} state that ``\emph{it is unsafe to release trained classifiers since valuable information about the training set can be extracted from them}''.
The problem that we address in this paper, \ie, to determine whether a sample has been used to train a given model, is related to the privacy implications of machine learning systems. 
They were first discussed in the context of support vector machines~\citep{rubinstein09learning,Biggio14securitysvm}.  
The problem of ``unintended memorization'' \citep{carlini18secret} appears in most applications of machine learning, such as natural language processing systems~\citep{carlini18secret} or image classification~\cite{yeom18privacyrisk}.

More specifically, we consider the problem of membership inference, \ie, we aim at determining if a specific image was used to train a model, given only the (image, label) pair and the model parameters. 
This question is important to protect both the privacy and intellectual property associated with data. 
For neural networks, the privacy issue was recently considered by \citet{yeom18privacyrisk} for the MNIST and CIFAR datasets. 
The authors evidence the close relationship between overfitting and privacy of training images. 
This is reminiscent of prior membership inference attacks, which  employ the output of the classifier associated with a particular sample to determine whether it was used during training or not~\citep{shokri16membershipinference}.  

At this stage, it is worth defining the different levels of information to which the ``attacker'', \ie, the membership inference system, has access to. 
We assume that the attacker knows the data distribution and the specifications of the model (training procedure, architecture of the network, etc), even though they are not necessarily required for all methods. 
We refer to the \emph{white-box} setting as the case where the attacker knows all the network parameters. 
On a side note, the setup commonly adopted in differential privacy~\cite{dwork2006calibrating} corresponds to the white-box setting, where the attacker additionally knows all the training samples except the one to be tested. 

The \emph{black-box} setting is when these parameters are unknown.
For classification models, the attacker has only access to the output for a given input, in one of the following forms:  

~\hfill \begin{minipage}{0.95\linewidth}
\emph{(i)} the classifier decision; \newline
\emph{(ii)} the loss of the correct label;\newline
\emph{(iii)} the full response for all classes. %
\end{minipage}

Prior works on membership inference commonly assume (i) or (iii). 
Our paper focuses on the black-box case (ii), in which we know the loss incurred by the correct label. %
The state of the art in this setting are the shadow models proposed by~\citet{shokri16membershipinference}. 

In our work, we use a probabilistic framework to derive a formal analysis of the optimal attack. 
This framework encompasses both Bayesian learning, and noisy training, where the noise is injected~\citep{welling2011langevin} or comes from the stochasticity of SGD. 
Under mild assumptions on the distribution of the parameters, we derive the optimal membership inference strategy. 
This strategy only depends on the classifier through evaluation of the loss, thereby showing that black-box attacks will perform as well as white-box attacks in this optimal asymptotic setting.
This result may explain why, to the best of our knowledge, the literature does not report white-box attacks outperforming the state-of-the-art black-box-(ii) attacks. 

The aforementioned optimal strategy is not tractable, therefore we introduce approximations to derive an explicit method for membership inference. 
As a byproduct of this derivation, we show that state-of-the-art \makebox{approaches~\citep{shokri16membershipinference,yeom18privacyrisk}} are coarser approximations of the optimal strategy. 
One of the approximation drastically simplifies the membership inference procedure by simply relying on the loss and a calibration term. 
We employ this strategy to the more complex case of neural networks, and show that it outperforms all approaches we are aware of. %

In summary, our main contributions are as follows: 
\begin{itemize}%
\item We show, under a few assumptions on training, that the optimal inference only depends on the loss function, and not on the parameters of the classifier. In other terms, white-box attacks don't provide any additional information and result in the same optimal strategy. 

\item We employ different approximations to derive three explicit membership attack strategies. We show that state-of-the-art methods constitute other approximations. Simple simulations show the superiority of our approach on a simple regression problem. 
\item We apply a simplified, tractable, strategy to infer the membership of images to the train set in the case of the public image classification benchmarks CIFAR and Imagenet. It outperforms the state of the art for membership inference, namely the shadow models. 
\end{itemize}

The paper is organized as follows. Section~\ref{sec:related} reviews related work. 
Section~\ref{sec:model} introduces our probabilistic formulation and derives our main theoretical result. %
This section also discusses the connection between membership inference and differential privacy. 
Section~\ref{sec:practical} considers approximations for practical use-cases, which allow us to derive inference strategies in closed-form, some of which are connected with existing methods from the literature. 
Section~\ref{sec:algorithms} summarizes the practical algorithms derived from our analysis. 
Finally, Section~\ref{sec:experiments} considers the more difficult case of membership inference for real-life neural networks and datasets. %

\section{Related work}
\label{sec:related}

\mypar{Leakage of information from the training set. }
Our work is related to the topics of overfitting and memorization capabilities of classifiers.  
Determining what neural networks actually memorize from their training set is not trivial. A few recent works~\citep{zhang16understanding,yeom18privacyrisk} evaluate how a network can fit random labels. 
\citet{zhang16understanding} replace true labels by random labels and show that popular neural nets can perfectly fit them in simple cases, such as small datasets (CIFAR10) or Imagenet without data augmentation. 
\citet{krueger17memorization} extend their analysis and argue in particular that the effective capacity of neural nets depends on the dataset considered. 
In a privacy context, \citet{yeom18privacyrisk} exploit this memorizing property to watermark networks. 
As a side note, random labeling and data augmentation have been used for the purpose of training a network without any annotated data~\citep{dosovitskiy2014discriminative,bojanowski2017unsupervised}.  

In the context of differential privacy~\citep{dwork2006calibrating}, recent works~\citep{wang2016average,bassily2016algorithmic} suggest that guaranteeing privacy requires learning systems to generalize well, \ie, to not overfit. 
\citet{wang2015privacyforfree} show that Bayesian posterior sampling offers differential privacy guarantees. 
\citet{abadi16deep} introduce noisy SGD to learn deep models with differential privacy.

\vspace{-1em}
\mypar{Membership Inference.} A few recent works~\citep{,hayes17logan,shokri16membershipinference,long18understanding} have addressed membership inference. 
\citet{yeom18privacyrisk} propose a series of membership attacks and derive their performance. %
\citet{long18understanding} observe that some training images are more vulnerable than others and propose a strategy to identify them. %
\citet{hayes17logan} analyze privacy issues arising in generative models.
\citet{dwork2015traceability} and \citet{sankararaman2009genomicprivacy} provide optimal strategies for membership inference in genomics data.

\vspace{-1em}
\mypar{Shadow models.} 

Shadow models were introduced by \citet{shokri16membershipinference} in the context of black-box attacks. 
In this setup, an attacker has black-box-(iii) access (full response for all classes) to a model trained on a private dataset, and to a public dataset that follows the same distribution as the private dataset. 
The attacker wishes to perform membership inference using black-box outputs of the private model. 
For this, the attacker simulates models by training \emph{shadow models} on known splits from the public set. 
On this simulated models, the attacker can analyze the output patterns corresponding to samples from the training set and from a held-out set.
\citet{shokri16membershipinference} propose to train an attack model that learns to predict, given an output pattern, whether it corresponds to a training or held-out sample.
If the attack model simply predicts ``training" when the output activations fire on the correct class, this strategy is equivalent to \citet{yeom18privacyrisk}'s adversary. 
\citet{salem2019leaks} further show that shadow models work under weaker assumptions than those of \citet{shokri16membershipinference}.

\section{Membership inference model}
\label{sec:model}

In this section, we derive the Bayes optimal performance for membership inference (Theorem~\ref{thm:optimal_strategy}). 
We then make the connection with differential privacy and propose looser guarantees that prevent membership inference. %

\subsection{Posterior distribution of parameters}

Let $ \dist $ be a data distribution, from which we sample $ n \in \N $ points $ z_1, z_2, \dots, z_{n} {\iid} \dist$.
A machine learning algorithm produces parameters $ \theta $ that incur a low loss $ \sum_{i=1}^n  \ell(\theta, z_i)$. 
Typically in the case of classification, $z = (x, y)$ where $x$ is an input sample and $y$ a class label, and the loss function $\ell$ is high on samples $(x, y')$ for which $y' \neq y$.

We assume that the machine learning algorithm has some randomness, and we model it with a posterior distribution over parameters $ \theta | z_1, \dots, z_n $.
The randomness in $ \theta $ either comes from the training procedure (e.g., Bayesian posterior sampling), or arises naturally, as is the case with Stochastic Gradient methods.

In general, we assume that the posterior distribution follows: 
\begin{align}
\label{eq:temperature}
\prob(\theta~|~z_1, \dots, z_n) &\propto e^{ - \frac{1}{T} \sum_{i=1}^n \ell(\theta, z_i) },
\end{align}

where $ T $ is a temperature parameter, intuitively controlling the stochasticity of $ \theta $. 
$ T = 1 $ corresponds to the case of the Bayesian posterior, $ T \to 0 $ the case of MAP (Maximum A Posteriori) inference and a small T to the case of averaged SGD \citep{polyak1992averaging}.
Note that we do not include a prior on $ \theta $: we assume that the prior is either uniform on a bounded $ \theta $, or that it has been factored in the loss term $ \ell $.

\subsection{Membership inference}

Given $ \theta $ produced by such a machine learning algorithm, membership inference asks the following question: What information does $ \theta $ contain about its training set $ z_1, \dots, z_n $?

Formally, we assume that binary membership variables $ m_1, m_2, \dots,  m_{n}$ 
are drawn independently, with probability $\lambda = \prob(m_i=1)$. 
The samples for which $ m_i = 0 $ are the test set, while the samples for which $ m_i = 1 $ are the training set. 
Equation~(\ref{eq:temperature}) becomes 
\begin{align}
\label{eq:temperaturem}
\prob(\theta~|~z_1, \dots, z_n, m_1, \dots, m_n) &\propto e^{ - \frac{1}{T} \sum_{i=1}^n m_i\ell(\theta, z_i) },
\end{align}

Taking the case of $ z_1 $ without loss of generality, membership inference determines, given parameters $ \theta $ and sample $ z_1 $, whether $ m_1 = 1 $ or $ m_1 = 0$. 
\begin{definition}[Membership inference]
Inferring the membership of sample $z_1$ to the training set amounts to computing:
\begin{align}
\mathcal{M}(\theta, z_1) := \prob( m_1 = 1~|~\theta, z_1).
\end{align}
\end{definition}

{\bf Notation.}
We denote by $ \sigma $ the sigmoid function $ \sigma(u) = (1 + e^{-u})^{-1} $. %
We collect the knowledge about the other samples and their memberships into the set $\ZM = \{z_2,...,z_n, m_2,..., m_n\}$.

\subsection{Optimal membership inference}

In Theorem \ref{thm:optimal_strategy}, we derive the explicit formula for $\mathcal{M}(\theta, z_1)$.

\begin{theorem}
\label{thm:optimal_strategy}
Given a parameter $ \theta $ and a sample $ z_1 $,
the optimal membership inference is given by:
{\small
\begin{align}
\mathcal{M}(\theta, z_1) &= \esp_{\ZM} \left[  \sigma \left( \log  \left(
\frac{\prob(\theta ~|~ {m_1=1, z_1, \ZM})}{\prob(\theta ~|~ {m_1=0, z_1, \ZM})} 
\right) + t_{\lambda}
\right) \right],
\end{align}
}%
with $ t_{\lambda} = \log \left( \frac{\lambda}{1- \lambda} \right) $.
\end{theorem}
\begin{proof}

By the law of total expectation, we have:
\begin{align}
\mathcal{M}(\theta, z_1)  &= \prob( m_1 = 1~|~\theta, z_1) \\
&=  \esp_{\ZM} \left[ \prob \left( m_1 = 1 ~|~ \theta, z_1, \ZM \right) \right].
\end{align}

Applying Bayes' formula: 
{\small
\begin{align}
 \prob \left( m_1 = 1 ~|~ \theta, z_1, \ZM \right) & = \frac{ \prob \left( \theta ~|~ m_1=1, z_1, \ZM \right) \prob( m_1 = 1) } {\prob \left( \theta ~|~ z_1, \ZM \right)},\\
 &= \frac{\alpha}{\alpha + \beta} 
 = \sigma \left( \log \left( \frac{\alpha}{\beta} \right) \right)
\label{eq:probratio}
 \end{align}}
where:
 \begin{align}
 \alpha &:=  \prob \left( \theta ~|~ m_1=1, z_1, \ZM \right) \prob( m_1 = 1) \\
 \beta  &:= \prob \left( \theta ~|~ m_1=0, z_1, \ZM  \right) \prob( m_1 = 0)
 \end{align}
Given that $\prob(m_1=1) = \lambda$,
 \begin{align}
 \log \left( \frac{\alpha}{\beta} \right) &=  \log  \left( \frac{\prob(\theta ~|~ m_1=1, z_1, \ZM)}{\prob(\theta ~|~ m_1=0, z_1, \ZM)} \right) + \log \left( \frac{\lambda}{1 - \lambda} \right), 
 \end{align}
which gives the expression for $\mathcal{M}(\theta, z_1)$.
\end{proof}

Note that Theorem \ref{thm:optimal_strategy} only relies on the fact that $ \theta $ given $\{z_1,..., z_n, m_1, ..., m_n\}$ is a random variable, but it does not make any assumption on the form of the distribution. %
In particular the loss $ \ell $ does not appear in the expression.

Theorem \ref{thm:optimal_loss} uses the assumption in Equation (\ref{eq:temperaturem}) to further explicit $ \mathcal{M}(\theta, z_1) $; we give its formal expression below, prove it, and analyze the expression qualitatively.
Let us first define the posterior over the parameters given samples $ z_2, \dots, z_{n} $ and memberships $ m_2, \dots, m_{n} $:
\begin{align}
p_{\ZM}(\theta) :=  \frac{ e^{ - \frac{1}{T} \sum_{i=2}^{n} m_i \ell(\theta, z_i) } }{\int_{t}  e^{ - \frac{1}{T} \sum_{i=2}^{n} m_i \ell(t, z_i) } dt}. 
\end{align}

\begin{theorem}
\label{thm:optimal_loss}
Given a parameter $ \theta $ and a sample $ z_1 $,
the optimal membership inference is given by:
\begin{align}
\label{eq:expectsigma}
\mathcal{M}(\theta, z_1) &= \esp_{\ZM} \left[  \sigma \left( s(z_1, \theta, p_{\ZM}) + t_{\lambda} \right) \right]
\end{align}
where we define the following score: %
\begin{align} 
\tau_p(z_1) &:= - T \log \left( \int_{t} e^{ -\frac{1}{T}\ell(t, z_1) } p(t) dt \right) 
\label{eq:taudefinition}
\\
s(z_1, \theta, p) &:=  \frac{1}{T} \left( \tau_p(z_1) - \ell (\theta, z_1) \right). %
\label{eq:sdefinition}
\end{align}

\end{theorem}

\begin{proof}

Singling out $m_1$ in Equation~(\ref{eq:temperaturem}) yields the following expressions for $\alpha$ and $\beta$: 
\begin{align}
\alpha &= \lambda 
\frac{ e^{- \frac{1}{T} \ell(\theta, z_1)} e^{ - \frac{1}{T} \sum_{i=2}^{n} m_i \ell(\theta, z_i) } }
     {\int_{t} e^{- \frac{1}{T}  \ell(t, z_1)}  e^{ -\frac{1}{T}\sum_{i=2}^{n} m_i \ell(t, z_i) }dt} \\
&= \lambda \frac{ e^{- \frac{1}{T}  \ell(\theta, z_1)} p_{\ZM}(\theta)}{\int_{t} e^{- \frac{1}{T}  \ell(t, z_1)} p_{\ZM}(t) dt},
\end{align}
and
\begin{align}
\beta &= (1-\lambda) \frac{e^{ - \frac{1}{T} \sum_{i=2}^{n} m_i \ell(\theta, z_i) } }{\int_{t} e^{ - \frac{1}{T} \sum_{i=2}^{n} m_i \ell(t, z_i) } dt} = (1-\lambda) p_{\ZM}(\theta).
\end{align}
Thus,
\begin{align}
\log \left( \frac{\alpha}{\beta} \right)  &= - \frac{\ell (\theta, z_1)}{T} 
- \log \left( \int_{t} e^{- \frac{1}{T}\ell(t, z_1))} p_{\ZM}(t) dt \right) + t_\lambda \label{eq:logsumexp} \nonumber \\
&=s(z_1, \theta, p_{\ZM}) + t_\lambda.
\end{align}

Then, Equation~(\ref{eq:probratio}) yields the expected result.
\end{proof}

The first observation in Theorem \ref{thm:optimal_loss} is that $ \mathcal{M}(\theta, z_1) $ does not depend on the parameters $ \theta $ beyond the evaluation of the loss $ \ell(\theta, \cdot)$: this strategy does not require, for instance, internal parameters of the model that a white-box attack could provide.
This means that if we can compute $\tau_p$ or approximate it well enough, then the optimal membership inference depends only on the loss. In other terms, asymptotically,
\textbf{the white-box setting does not provide any benefit compared to black-box membership inference.}

Let us analyze qualitatively the terms in the expression: %
Since $ \ZM $ is a training set, $ p_{\ZM} $ corresponds to a posterior over this training set, \ie, a typical distribution of trained parameters. 
$\tau_p(z_1)$ is the softmin of loss terms $ \ell(\cdot, z_1) $ over these typical models, and corresponds therefore to the {\em typical loss of sample $ z_1 $} under models that have not seen $ z_1 $.

The quantity $\tau_p(z_1)$ can be seen as a threshold, to which the loss $ \ell(\theta, z_1) $ is compared.
Around this threshold, when $ \ell(\theta, z_1) \approx \tau_p(z_1)$, then $ s \approx 0 $: since $ \sigma(t_{\lambda}) = \lambda $, the membership posterior probability $ \mathcal{M}(\theta, z_1)  $ is equal to $ \lambda $, and thus we have no information on $ m_1 $ beyond prior knowledge.
As the loss $ \ell(\cdot, z_1) $ gets lower than this threshold, $s$ becomes positive.
Since $ \sigma $ is non decreasing, when $s(z_1, \theta, p_{\ZM}) > 0 $, $ \mathcal{M}(\theta, z_1) > \lambda $ and thus we gain non-trivial membership information on $ z_1 $.

Another consequence of Theorem \ref{thm:optimal_loss} is that a higher temperature $ T $ decreases $ s $, and thus decreases $ \prob( m_1 = 1~|~\theta, z_1)  $: it corresponds to the intuition that more randomness in $ \theta $ protects the privacy of training data.

\subsection{Differential privacy and guarantees}

In this subsection we make the link with differential privacy. 

Differential privacy \citep{dwork2006calibrating} is a framework that allows to learn model parameters $ \theta $ while maintaining the confidentiality of data.
It ensures that even if a malicious attacker knows parameters $ \theta $ and samples $ z_i $, $ i \geq 2$, for which $ m_i = 1$, the privacy of $ z_1 $ is not compromised. 

\begin{definition}[$\epsilon$-differential privacy]
A machine learning algorithm is $ \epsilon $-differentially private if, for any choice of $ z_1$ and $\ZM$, 
\begin{align}
\label{eq:epsdifpriv}
\log \left( \frac{\prob(\theta ~|~ {m_1 =1, z_1, \ZM})}{\prob(\theta ~|~ {m_1=0, z_1, \ZM})} \right) < \epsilon.
\end{align}
\end{definition}

Note that this definition is slightly different from the one of \citet{dwork2006calibrating} in that we consider the removal of $ z_1 $ rather than its substitution with $ z' $. Additionally we consider probability densities instead of probabilities of sets, without loss of generality.

\begin{property}[$\epsilon$-differential privacy]
If the training is $ \epsilon $-differentially private, then: 
\begin{align}
\prob( m_1 = 1~|~\theta, z_1) &\leq \lambda + \frac{\epsilon}{4}.  \label{eq:diff_proba}
\end{align}
\end{property}

\begin{proof}
Combining Equation~(\ref{eq:epsdifpriv}) and the fact that $ \sigma(u) \leq \sigma(v) + \max(u-v, 0) / 4 $ (Appendix~\ref{sec:boundedsigma}), we have:
\begin{align}
  \sigma \left( \log  \left(
\frac{\prob(\theta ~|~ m_1=1, z_1, \ZM)}{\prob(\theta ~|~  m_1=0, z_1, \ZM)} \right) + t_{\lambda} \right) &\leq \sigma(t_{\lambda}) + \frac{\epsilon}{4} \nonumber \\
&= \lambda + \frac{\epsilon}{4}. 
\end{align}
Combining this expression with Theorem~\ref{thm:optimal_strategy} yields the result.
\end{proof}
Note that this bound gives a tangible sense of $ \epsilon $. 
In general, decreasing $ \epsilon $ increases privacy, but there is no consensus over ``good" values of $ \epsilon $; this bound indicates for instance that $ \epsilon = 0.01 $ would be sufficient for membership privacy.

$ \epsilon $-differential privacy gives strong membership inference guarantees, at the expense of a constrained training procedure resulting generally in a  loss of accuracy \cite{abadi16deep}.
However, if we assume that the attacker knows the $ z_i, i \geq 2 $ for which $m_i = 1$, $ \epsilon $-differential privacy is required to protect the privacy of $ z_ 1$. 
Depending on the information we have on $ z_i, i \geq 2 $, there is a continuum between differential privacy (all $ z_i$'s are known) and membership inference (only prior knowledge on $ z_i $).
In the case of membership inference, it suffices to have the following guarantee:

\begin{definition}[$(\epsilon, \delta)$ membership privacy]
The training is $(\epsilon, \delta)$-membership private  for
some $ \epsilon>0, \delta>0$ if with probability $ 1 - \delta $ over the choice of $ \ZM$:
\begin{align}
\int_{t} \ell(t, z_1) p_{\ZM}(t) dt - \ell (\theta, z_1) \leq \epsilon. \label{eq:guarantee}
\end{align}
\end{definition}

\begin{property} If the training is $(\epsilon, \delta)$-membership private, then: 
\begin{align}
\prob( m_1 = 1~|~\theta, z_1) \leq \lambda + \frac{\epsilon}{4T} + \delta.
\end{align}
\end{property}

\begin{proof}
Jensen's inequality states that for any distribution $p$ and any function $f$:
\begin{align}
 \int_{t} f(t) p(t) dt &\leq \log \left( \int_{t} e^{f(t)} p(t) dt \right),
\end{align}
hence the score $s$ from Equation~(\ref{eq:sdefinition}) verifies:
\begin{align}
s(z_1, \theta, p) \leq \frac{1}{T}\left(\int_{t} \ell(t, z_1) p(t) dt - \ell (\theta, z_1) \right) .
\end{align}
Thus, distinguishing the cases $\delta$ and $1-\delta$ in the expectation in Equation~(\ref{eq:expectsigma}), 
\begin{align}
\prob( m_1 = 1~|~\theta, z_1) &\leq \delta + (1 - \delta) \left( \lambda + \frac{\epsilon}{4T} \right) \\
&\leq \lambda + \frac{\epsilon}{4T} + \delta,
\end{align}
which gives the desired bound. 
\end{proof}

Membership privacy provides a \emph{post-hoc} guarantee on $ \theta, z_1 $.
Guarantees in the form of Equation (\ref{eq:guarantee}) can be obtained by PAC (Probably Approximately Correct) bounds.
\section{Approximations for membership inference}
\label{sec:practical}

Estimating the probability of Equation (\ref{eq:sdefinition}) mainly requires to compute the term $\tau_p$. 
Since its expression is intractable, we use approximations to derive concrete membership attacks (MA). 
We now detail these approximations, referred to as MAST (MA Sample Threshold), MALT (MA Loss Threshold) and MATT (MA Taylor Threshold).

\begin{figure*}[t]
  \begin{minipage}{0.49\linewidth}
    \includegraphics[width=0.99\linewidth]{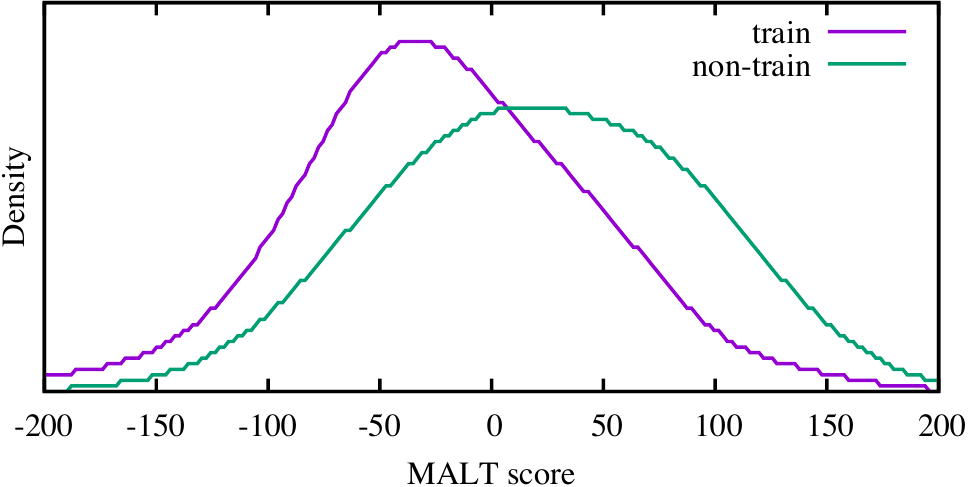}
  \end{minipage}
  \hfill
  \begin{minipage}{0.49\linewidth}
    \includegraphics[width=0.99\linewidth]{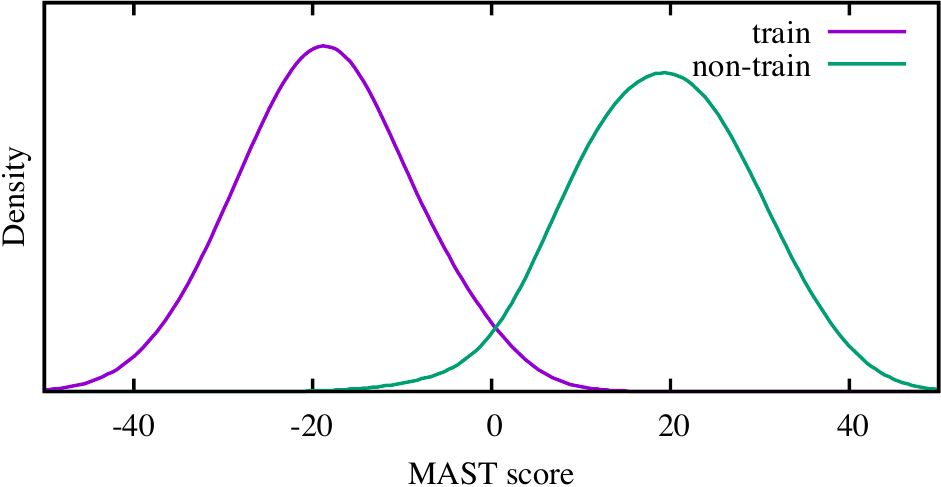}
  \end{minipage}
  \caption{\label{fig:gauss}
  Comparison of MALT and MALT for membership inference on the mean estimator for Gaussian data ($n=100$ samples in 2000 dimensions). 
  Distribution of scores $s$ used to distinguish between samples seen or not at training.
  \emph{MALT}: a single threshold is used for all the dataset;
  \emph{MAST}: each sample gets assigned a different threshold.
  MAST better separates training and non-training samples. 
  }
\end{figure*}

\subsection{MAST: Approximation of $\tau(z_1)$}

We first make the mean-field assumption that $ p_{\ZM}(t) $ does not depend on $\ZM$ (we note it $p$), and define
\begin{align}
\label{eq:deftauz}
\tau(z_1) :=  \log \left( \int_{t} e^{-\frac{1}{T}\ell(t, z_1)} p(t) dt \right).
\end{align}

The quantity $ \tau(\cdot) $ is a ``calibrating" term that reflects the difficulty of a sample. 
Intuitively, a low $ \tau(z_1) $ means that the sample $ z_1 $ is easy to predict, and thus a low value of $ \ell (\theta, z_1) $ does not necessarily indicate that $ z_1 $ belongs to the train set.
Thus, Theorem~\ref{thm:optimal_loss} gives the optimal attack model:
\begin{equation}
\boxed{s_\mathrm{MAST}(\theta, z_1) = - \ell(\theta, z_1)+ \tau (z_1). }
\end{equation}

\subsection{MALT: Constant $\tau$}

If we further assume that $ \tau(\cdot)$ is constant, the optimal strategy reduces to predicting that $ z_1 $ comes from the training set if its loss $ \ell (\theta, z_1) $ is lower than this threshold $ \tau $, and from the test set otherwise:
\begin{equation}
\label{eq:malt}
\boxed{s_\mathrm{MALT}(\theta, z_1) = - \ell(\theta, z_1)+ \tau.} 
\end{equation}

A similar strategy is proposed by \citet{yeom18privacyrisk} for Gaussian models.
\citet{carlini18secret} estimate a secret token in text datasets by comparing probabilities of the sentence ``My SSN is X" with various values of X.
Surprisingly, to the best of our knowledge, it has not been proposed in the literature to estimate the threshold $ \tau $ on public data, and to apply it for membership inference.
As we show in Section \ref{sec:experiments}, this simple strategy yields better results than shadow models.
Their attack models take as input the softmax activation $ \phi_{\theta}(x) $ and the class label $ y $, and predict whether the sample comes from the training or test set. 
For classification models, $ \ell(\theta, (x, y)) = - \log(\phi_{\theta}(x)_y) $. 
Hence the optimal attack performs:
\begin{equation}
s_\mathrm{MALT}(\theta, (x, y)) = \log(\phi_{\theta}(x)_y)+ \tau.
\end{equation}
In \citet{shokri16membershipinference}, we argue that the attack model essentially performs such an estimation, albeit in a non-explicit way.
In particular, we believe that the gap between \citet{shokri16membershipinference}'s method and ours is due to instabilities in the estimation of $ \tau $ and the numerical computation of the $ \log $, as the model is given only $ \phi_{\theta}(x)$.
As a side note, the expectation term in $ \ZM = {z_2, \dots, z_n, m_2, \dots, m_n} $ is very similar in spirit to the shadow models, and they can be viewed as a Monte-Carlo estimation of this quantity.
\vspace{-1em}
\mypar{An experiment with Gaussian data.}
We illustrate the difference between a MALT (global $ \tau $) and MAST (per-sample $ \tau(\cdot) $) on a simple toy example. 
Let's assume we estimate the mean $\mu$ of Gaussian data with unit variance.

We sample $ n $ values $ z _1, \dots, z_n $ from $ \dist = \mathcal{N}(\mu, I)$. 
The estimate of the mean is $\theta = \frac{1}{n'} \sum_{i=1}^n m_i z_i$ where $ n' = | \{i~|~m_i=1\} | $.
We have (see Appendix \ref{app:gauss} for derivations):
\begin{align}
\ell(\theta, z_i) &:= \frac{1}{2} \| z_i - \theta \|^2 \\
\tau(z_i) &= \frac{n'}{2(n'+1)} \| z_i - \mu\|^2 \\
\tau &= \frac{n'}{2(n'+1)} \esp \| z - \mu \|^2 = \frac{n'}{2(n'+1)} d. 
\end{align}

The expression of $\tau(z_i)$ shows that the ``difficulty" of sample $ z_i $ is its distance to $ \mu $, \ie, how untypical this sample is. %

Figure \ref{fig:gauss} shows the results with a global $ \tau $ or a per-sample $ \tau $: the per-sample $ \tau $ better separates the two distributions, leading to an increased membership inference accuracy.

\vspace{-1em}
\paragraph{MATT: Estimation with Taylor expansion}
\label{subsec:taylor}

We assume that the posterior induced by the loss $ \mathcal{L}(\theta) = \sum_{i=1}^n m_i\ell(\theta, z_i)$ is a Gaussian centered on $ \theta^* $, the minimum of the loss, with covariance $ C $.
This corresponds to the Laplace approximation of the posterior distribution. 
The inverse covariance matrix $ C^{-1} $ is asymptotically $ n $ times the Fisher matrix \cite{vandervaart98asymptoticstatistics}, which itself is the Hessian of the loss~\makebox{\cite{kullback97informationtheory}}:
\begin{align}
\prob(\theta~|~z_1, \ZM) &= \frac{1}{\sqrt{\det{2\pi H^{-1}}}} e^{-\frac{1}{2} (\theta - \theta^*)^T H (\theta - \theta^*)}. 
\end{align}
We denote by $\theta_0^*$ (resp. $\theta_1^*$) the mode of the Gaussian corresponding to $ \{z_2,...,z_n\} $ (resp. \{$ z_1,...,z_n$\}), and $ H_0 $ (resp. $ H_1 $) the corresponding Hessian matrix.
We assume that $ H $ is not impacted by removing $ z_1 $ from the training set, and thus $ H := H_0 \approx H_1 $ (cf. appendix \ref{app:hessian} for a more precise justification).
The log-ratio is therefore 
{\small
\begin{align}
\log & \left( \frac{\prob(\theta ~|~ m_1 = 1, z_1, \ZM)}{\prob(\theta ~|~ m_1 = 0, z_1, \ZM)} \right)  \label{eq:hessian_dl} \\
&= - \frac{1}{2} (\theta - \theta_1^*)^T H (\theta - \theta_1^*) +  \frac{1}{2} (\theta - \theta_0^*)^T H (\theta - \theta_0^*) \nonumber \\
&= (\theta - \theta_0^*)^T H (\theta_1^* - \theta_0^*) -  \frac{1}{2} (\theta_1^* - \theta_0^*)^T H (\theta_1^* - \theta_0^*). \nonumber 
\end{align}
}%
The difference $ \theta_1^* - \theta_0^* $ can be estimated using a Taylor expansion of the loss gradient  around $ \theta_0^*$ (see \eg \citet{koh2017understanding}):
\begin{align}
\theta_1^* - \theta_0^* \approx - H^{-1} \nabla_{\theta} \ell(\theta_0^*, z_1)
\end{align}
Combining this with Equation (\ref{eq:hessian_dl}) leads to
{\small
\begin{align}
\label{eq:twoterm}
- (\theta - \theta_0^*)^T  \nabla_{\theta} \ell(\theta_0^*, z_1) -  \frac{1}{2} \nabla_{\theta} \ell(\theta_0^*, z_1)^T H^{-1} \nabla_{\theta} \ell(\theta_0^*, z_1). 
\end{align}
}
We study the asymptotic behavior of this expression when $n\rightarrow \infty$. 
On the left-hand side, the parameters $\theta$ and $\theta_0^*$ are estimates of the optimal $\theta^*$, and under mild conditions, the error of the estimated parameters is of order $1/\sqrt{n}$. 
Therefore the difference $ \theta - \theta_0^* $ is of order $ 1 / \sqrt{n}$. 
On the right-hand side, the matrix $H$ is the summation of $n$ sample-wise Hessian matrices. 
Therefore, asymptotically, the right-hand side shrinks at a rate $1/n$, which is negligible compared to the other, which shrinks at $1/\sqrt{n}$. 
In addition to the asymptotic reasoning, we verified this approximation experimentally. 
Thus, we approximate Equation~(\ref{eq:twoterm}) to give the following score:
\begin{align}
\label{eq:dl_grad}
\boxed{
s_\mathrm{MATT}(\theta, z_1) = - (\theta - \theta_0^*)^T  \nabla_{\theta} \ell(\theta_0^*, z_1). }
\end{align}

Equation (\ref{eq:dl_grad}) has an intuitive interpretation: parameters~$ \theta $ were trained using $ z_1 $ if their difference with a set of parameters trained without $ z_1 $ (i.e. $\theta_0^*$) is aligned with the direction of the update $- \nabla_{\theta} \ell(\theta_0^*, z_1) $.

\section{Membership inference algorithms}
\label{sec:algorithms}

In this section, we detail how the approximations of $ s(\theta, z_1, p) $ are employed to perform membership inference. 

We assume that a machine learning model has been trained, yielding parameters $\theta$. 
We assume also that similar models can be re-trained with different training sets.
Given a sample $z_1$, we want to decide whether $z_1$ belongs to the training set.

\subsection{The 0-1 baseline}
We consider as a baseline the ``0-1" heuristic, which predicts that $z_1$ comes from the training set if the class is predicted correctly, and from the test set if the class is predicted incorrectly.
We note $p_{\text{train}} $ (resp. $ p_{\text{test}}$) the classification accuracy on the training (resp. held-out) set. 
The accuracy of the heuristic is (see Appendix~\ref{app:probaderiv} for derivations):
\begin{equation}
p_{\text{bayes}}  = \lambda p_{\text{train}} + (1-\lambda) (1-p_{\text{test}}).
\end{equation}

For example when $\lambda=1/2$, since $p_{\text{train}} \geq p_{\text{test}}$ this heuristic is better than random guessing (accuracy $1/2$) and the improvement is proportional to the overfitting gap $p_{\text{train}} - p_{\text{test}}$.
%

\subsection{Making hard decisions from scores}
\label{sec:scorethesh}
Variants of our method provide different estimates of $ s(\theta, z_1, p) $.
Theorem \ref{thm:optimal_loss} shows that this score has to be passed through a sigmoid function, but since it is an increasing function, the threshold can be chosen directly on these scores. %
Estimation of this threshold has to be conducted on simulated sets, for which membership information is known.
We observed that there is almost no difference between chosing the threshold on the set to be tested and cross-validating it. 
This is expected, as a one-dimensional parameter the threshold is not prone to overfitting.

\subsection{Membership algorithms}
\label{subsec:membershipalgorithms}

\textbf{MALT: Threshold on the loss.}
Since $\tau$ in Equation (\ref{eq:malt}) is constant, and using the invariance to increasing functions, we need only to use loss value for the sample, $\ell(\theta, z_1)$. 

\textbf{MAST: Estimating $\tau(z_1)$.}
To estimate $\tau(z_1)$ in Equation~(\ref{eq:deftauz}), we train several models with different subsamples of the training set. 
This yields a set of per-sample losses for $z_1$ that are averaged into an estimate of $\tau(z_1)$. 

\textbf{MATT: the Taylor approximation.}
We run the training on a separate set to obtain $\theta_0^*$. 
Then we take a gradient step over the loss to estimate the approximation in Equation~(\ref{eq:dl_grad}).
Note that this strategy is not compatible with neural networks because the assumption that parameters lie around a unique global minimum does not hold. 
In addition, parameters from two different networks $ \theta $ and $ \theta_0^* $ cannot be compared directly as neural networks that express the same function can have very different parameters (\eg because channels can be permuted arbitrarily).

\section{Experiments}
\label{sec:experiments}

In this section we evaluate the membership inference methods on machine-learning tasks of increasing complexity.

\subsection{Evaluation}

We evaluate three metrics: the accuracy of the attack, and the mean average precision when detecting either from train ({mAP\textsubscript{train}}) or test ({mAP\textsubscript{test}}) images. 
For the mean average precision, the scores need not to be thresholded, the metric is invariant to mapping by any increasing function.

\begin{table}[t]
\caption{
	\label{tab:expe_cifar_accu}
	Accuracy (top) and mAP (bottom) of membership inference on the 2-class logistic regression with simple CNN features, for different types of attacks. Note that 0-1 corresponds to the baseline~\cite{yeom18privacyrisk}. 
	We do not report ~\citet{shokri16membershipinference} since Table \ref{tab:shadow_comparison} shows MALT performs better.
	Results are averaged over $ 100 $ different random seeds.
}
\begin{center}
\begin{tabular}{@{}l|rr|rrr@{}}
\toprule
            & \multicolumn{2}{c|}{Model accuracy} & \multicolumn{3}{c}{Attack accuracy}\\
n	     		& train		& validation	& $0-1$	& MALT	& MATT \\ \midrule
400			& 97.9		& 93.8		& 52.1	& 54.4	& \bf{57.0} \\
1000			& 97.3		& 94.5		& 51.4	& 52.6	& \bf{54.5} \\
2000			& 96.8		& 95.2		& 50.8	& 51.7	& \bf{53.0} \\
4000			& 97.7		& 95.6		& 51.0	& 51.4	& \bf{52.1} \\
6000			& 97.5		& 96.0		& 50.7	& 51.0	& \bf{51.8} \\
\bottomrule
\end{tabular}

\vspace{6mm}

\begin{tabular}{@{}lrrrrr@{}}
\toprule
& \multicolumn{2}{c}{mAP\textsubscript{test}}	& \phantom{c}	& \multicolumn{2}{c}{mAP\textsubscript{train}} \\
\cmidrule{2-3} \cmidrule{5-6}
n 		& MALT 	& MATT 	&& MALT 	& MATT \\ \midrule
400		& 55.8	& 60.1	&& 51.9	& 57.1  \\
1000		& 53.2	& 56.6	&& 50.5	& 54.8 \\
2000		& 51.8	& 54.4	&& 50.4	& 53.4 \\
4000		& 51.9	& 53.7	&& 50.1	& 52.6 \\
6000		& 51.4	& 53.0	&& 50.2	& 52.2 \\
\bottomrule
\end{tabular}
\end{center}
\end{table}

\subsection{Logistic regression}

CIFAR-10 is a dataset of $ 32 \times 32 $ pixel images grouped in 10 classes.
In this subsection, we consider two of the classes (\emph{truck} and \emph{boat}) and vary the number of training images from $n=400$ to $6,000$. 

We train a logistic regression to separate the two classes. 
The logistic regression takes as input features extracted from a pretrained Resnet18 on CIFAR-100 (a disjoint dataset).
The regularization parameter $ C $ of the logistic regression is cross-validated on held-out data.

We assume that $\lambda=1/2$ ($n/2$ training images and $n/2$ test images). 
We also reserve $n/2$ images to estimate $\theta_0^*$ for the MATT method.
In both experiments, we report the peak accuracy obtained for the best threshold (cf. Section~\ref{sec:scorethesh}).

Table \ref{tab:expe_cifar_accu} shows the results of our experiments, in terms of accuracy and mean average precision. 
In accuracy, the Taylor expansion method MATT outperforms the MALT method, for any number of training instances $n$, which itself obtains much better results than the naive 0-1 attack.

Interestingly, it shows a difference between MALT and MATT: both perform similarly in terms of mAP\textsubscript{test}, but MATT slightly outperforms MALT in mAP\textsubscript{train}.
The main reason for this difference is that the MALT attack is asymmetric: it is relatively easy to predict that elements come from the test set, as they have a high loss, but elements with a low loss can come either from the train set or the test set.
\subsection{Small convolutional network}

In this section we train a small convolutional network\footnote{
	2 convolutional and 2 linear layers with Tanh non-linearity.
} with the same architecture as \citet{shokri16membershipinference} on the CIFAR-10 dataset, using a training set of $15,000$ images. 
Our model is trained for $ 50 $ epochs with a learning rate of $ 0.001 $.
We assume a balanced prior on membership ($\lambda = 1/2$). %

We run the MALT and MAST attacks on the classifiers. 
As stated before, the MATT attack cannot be carried out on convolutional networks.
For MAST, the threshold is estimated from 30 shadow models: we train these 30 shadow models on $ 15,000 $ images chosen among the train+held-out set ($30,000$ images).
Thus, for each image, we have on average  15  models trained on it and  15  models not trained on it: we estimate the threshold for this image by taking the value $ \tau(z) $ that separates the best the two distributions: this corresponds to a non-parametric estimation of $ \tau(z) $.

Table~\ref{tab:shadow_comparison} shows that our estimations outperform the related works. 
Note that this setup gives a slight advantage to MAST as the threshold is estimated directly for each sample under investigation, whereas MALT first estimates a threshold, and then applies it to never-seen data.
Yet, in contrast with the experiment on Gaussian data, MAST performs only slightly better than MALT. 
Our interpretation for this is that the images in the training set have a high variability, so it is difficult to obtain a good estimate of $\tau(z_1)$. 
Furthermore, our analysis of the estimated thresholds $ \tau(z_1) $ show that they are very concentrated around a central value $ \tau $, so their impact when added to the scores is limited. 

Therefore, in the following experiment we focus on the MALT attack.

\begin{table}[t]
\centering
\caption{\label{tab:shadow_comparison}
	Accuracy of membership attacks on the CIFAR-10 classification with a simple neural network. 
	The numbers for the related works are from the respective papers. \medskip
}
\begin{tabular}{@{}lr@{}}
\toprule
Method										&	Accuracy	\\ \midrule
0-1~\cite{yeom18privacyrisk}						&	69.4		\\
Shadow models~\cite{shokri16membershipinference}	&	73.9		\\
MALT										&	\bf{77.1}	\\
\midrule
MAST										&	77.6		\\
\bottomrule
\end{tabular}
\vspace{-5pt}
\end{table}

\subsection{Evaluation on Imagenet}
\label{sec:baselines}

We evaluate a real-world dataset and tackle classification with large neural networks on the Imagenet dataset~\citep{deng2009imagenet,russakovsky15imagenet}, which contains 1.2 million images partitioned into 1000 categories. 
We divide Imagenet equally into two splits, use one for training and hold out the rest of the data.

We experiment with the popular VGG-16 \citep{simonyan15vgg} and Resnet-101 \citep{he16resnet} architectures.
The model is learned in $ 90 $ epochs, with an initial learning rate of $ 0.01 $, divided by $ 10 $ every $ 30 $ epochs. 
Parameter optimization is conducted with SGD with a momentum of $ 0.9 $, a weight decay of $ 10^{-4} $, and a batch size of $ 256 $.

We conduct the membership inference test by running the 0-1 attack and MALT.
An important factor for the success of the attacks is the amount of data augmentation. 
To assess the effect of data augmentation, we train different networks with varying data augmentation: None, Flip+Crop$\pm$5, Flip+Crop (by increasing intensity).

Table~\ref{tab:attack_baselines} shows that data augmentation reduces the gap between the training and the held-out accuracy. 
This decreases the accuracy of the Bayes attack and the MALT attack. 
As we can see, without data augmentation, it is possible to guess with high accuracy if a given image was used to train a model (about $90$\% with our approach, against $77$\% for existing approaches).
Stronger data augmentation reduces the accuracy of the attacks, that still remain above $ 64\%$.

\begin{table}[t]
\centering
\caption{\label{tab:attack_baselines}
Imagenet classification with deep convolutional networks: Accuracy of membership inference attacks of the models. \medskip
}
\begin{tabular}{@{}llcc@{}}
\toprule
Model     		& Augmentation 		& 0-1	& MALT \\ \midrule
Resnet101 	& None				& 76.3	& 90.4 \\ 
			& Flip, Crop $\pm5$		& 69.5	& 77.4 \\ 
			& Flip, Crop 			& 65.4	& 68.0 \\ \midrule
VGG16     	& None   				& 77.4	& 90.8  \\
    	 		& Flip, Crop $\pm5$		& 71.3	& 79.5  \\
     			& Flip, Crop			& 63.8	& 64.3  \\ \bottomrule
\end{tabular}
\end{table}

\section{Conclusion}
\label{sec:conclusion}

This paper has addressed the problem of membership inference by adopting a probabilistic point of view. 
This led us to derive the optimal inference strategy. 
This strategy, while not explicit and therefore not applicable in practice, does not depend on the parameters of the classifier if we have access to the loss. 
Therefore, a main conclusion of this paper is to show that, asymptotically, white-box inference does not provide more information than an optimized black-box setting. 

We then proposed two approximations that lead to three concrete strategies. 
They outperform competitive strategies for a simple logistic problem, by a large margin for our most sophisticated approach (MATT). 
Our simplest strategy (MALT) is applied to the more complex problem of membership inference from a deep convolutional network on Imagenet, and significantly outperforms the baseline.

\clearpage

\bibliography{egbib}

\begin{thebibliography}{29}
\providecommand{\natexlab}[1]{#1}
\providecommand{\url}[1]{\texttt{#1}}
\expandafter\ifx\csname urlstyle\endcsname\relax
  \providecommand{\doi}[1]{doi: #1}\else
  \providecommand{\doi}{doi: \begingroup \urlstyle{rm}\Url}\fi

\bibitem[Abadi et~al.(2016)Abadi, Chu, Goodfellow, McMahan, Mironov, Talwar,
  and Zhang]{abadi16deep}
Martin Abadi, Andy Chu, Ian Goodfellow, Brendan McMahan, Ilya Mironov, Kunal
  Talwar, and Li~Zhang.
\newblock Deep learning with differential privacy.
\newblock In \emph{CCS}, 2016.

\bibitem[Ateniese et~al.(2015)Ateniese, Mancini, Spognardi, Villani, Vitali,
  and Felici]{ateniese2015hacking}
Giuseppe Ateniese, Luigi~V Mancini, Angelo Spognardi, Antonio Villani, Domenico
  Vitali, and Giovanni Felici.
\newblock Hacking smart machines with smarter ones: How to extract meaningful
  data from machine learning classifiers.
\newblock \emph{IJSN}, 2015.

\bibitem[Bassily et~al.(2016)Bassily, Nissim, Smith, Steinke, Stemmer, and
  Ullman]{bassily2016algorithmic}
Raef Bassily, Kobbi Nissim, Adam Smith, Thomas Steinke, Uri Stemmer, and
  Jonathan Ullman.
\newblock Algorithmic stability for adaptive data analysis.
\newblock In \emph{STOC}, 2016.

\bibitem[Biggio et~al.(2014)Biggio, Corona, Nelson, Rubinstein, Maiorca,
  Fumera, Giacinto, and Roli]{Biggio14securitysvm}
Battista Biggio, Igino Corona, Blaine Nelson, Benjamin I.~P. Rubinstein, Davide
  Maiorca, Giorgio Fumera, Giorgio Giacinto, and Fabio Roli.
\newblock \emph{Security Evaluation of Support Vector Machines in Adversarial
  Environments}.
\newblock Springer International Publishing, 2014.

\bibitem[Bojanowski \& Joulin(2017)Bojanowski and
  Joulin]{bojanowski2017unsupervised}
Piotr Bojanowski and Armand Joulin.
\newblock Unsupervised learning by predicting noise.
\newblock In \emph{ICML}, 2017.

\bibitem[Carlini et~al.(2018)Carlini, Liu, Kos, Erlingsson, and
  Song]{carlini18secret}
Nicholas Carlini, Chang Liu, Jernej Kos, {\'U}lfar Erlingsson, and Dawn Song.
\newblock The secret sharer: Measuring unintended neural network memorization
  \& extracting secrets.
\newblock \emph{arXiv preprint arXiv:1802.08232}, 2018.

\bibitem[Deng et~al.(2009)Deng, Dong, Socher, Li, Li, and
  Fei-Fei]{deng2009imagenet}
Jia Deng, Wei Dong, Richard Socher, Li-Jia Li, Kai Li, and Li~Fei-Fei.
\newblock Imagenet: A large-scale hierarchical image database.
\newblock In \emph{CVPR}, 2009.

\bibitem[Dosovitskiy et~al.(2014)Dosovitskiy, Springenberg, Riedmiller, and
  Brox]{dosovitskiy2014discriminative}
Alexey Dosovitskiy, Jost~Tobias Springenberg, Martin Riedmiller, and Thomas
  Brox.
\newblock Discriminative unsupervised feature learning with convolutional
  neural networks.
\newblock In \emph{NIPS}, 2014.

\bibitem[Dwork et~al.(2006)Dwork, McSherry, Nissim, and
  Smith]{dwork2006calibrating}
Cynthia Dwork, Frank McSherry, Kobbi Nissim, and Adam Smith.
\newblock Calibrating noise to sensitivity in private data analysis.
\newblock In \emph{TCC}, 2006.

\bibitem[Dwork et~al.(2015)Dwork, Smith, Steinke, Ullman, and
  Vadhan]{dwork2015traceability}
Cynthia Dwork, Adam Smith, Thomas Steinke, Jonathan Ullman, and Salil Vadhan.
\newblock Robust traceability from trace amounts.
\newblock In \emph{Proceedings of the Symposium on the Foundations of Computer
  Science}, 2015.

\bibitem[Hayes et~al.(2017)Hayes, Melis, Danezis, and
  De~Cristofaro]{hayes17logan}
Jamie Hayes, Luca Melis, George Danezis, and Emiliano De~Cristofaro.
\newblock Logan: evaluating privacy leakage of generative models using
  generative adversarial networks.
\newblock \emph{arXiv preprint arXiv:1705.07663}, 2017.

\bibitem[He et~al.(2016)He, Zhang, Ren, and Sun]{he16resnet}
Kaiming He, Xiangyu Zhang, Shaoqing Ren, and Jian Sun.
\newblock Deep residual learning for image recognition.
\newblock In \emph{CVPR}, 2016.

\bibitem[Koh \& Liang(2017)Koh and Liang]{koh2017understanding}
Pang~Wei Koh and Percy Liang.
\newblock Understanding black-box predictions via influence functions.
\newblock In \emph{ICML}, 2017.

\bibitem[Krueger et~al.(2017)Krueger, Ballas, Jastrzebski, Arpit, Kanwal,
  Maharaj, Bengio, Fischer, Courville, Lacoste-Julien, and
  Bengio]{krueger17memorization}
David Krueger, Nicolas Ballas, Stanislaw Jastrzebski, Devansh Arpit,
  Maxinder~S. Kanwal, Tegan Maharaj, Emmanuel Bengio, Asja Fischer, Aaron
  Courville, Simon Lacoste-Julien, and Yoshua Bengio.
\newblock A closer look at memorization in deep networks.
\newblock In \emph{ICML}, 2017.

\bibitem[Kullback(1997)]{kullback97informationtheory}
S.~Kullback.
\newblock \emph{{Information Theory And Statistics}}.
\newblock Dover Publications, 1997.

\bibitem[Long et~al.(2018)Long, Bindschaedler, Wang, Bu, Wang, Tang, Gunter,
  and Chen]{long18understanding}
Yunhui Long, Vincent Bindschaedler, Lei Wang, Diyue Bu, Xiaofeng Wang, Haixu
  Tang, Carl~A Gunter, and Kai Chen.
\newblock Understanding membership inferences on well-generalized learning
  models.
\newblock \emph{arXiv preprint arXiv:1802.04889}, 2018.

\bibitem[Polyak \& Juditsky(1992)Polyak and Juditsky]{polyak1992averaging}
B.~T. Polyak and A.~B. Juditsky.
\newblock Acceleration of stochastic approximation by averaging.
\newblock \emph{SIAM J. Control Optim.}, 1992.

\bibitem[Rubinstein et~al.(2009)Rubinstein, Bartlett, Huang, and
  Taft]{rubinstein09learning}
Benjamin~IP Rubinstein, Peter~L Bartlett, Ling Huang, and Nina Taft.
\newblock Learning in a large function space: Privacy-preserving mechanisms for
  {SVM} learning.
\newblock \emph{arXiv:0911.5708}, 2009.

\bibitem[Russakovsky et~al.(2015)Russakovsky, Deng, Su, Krause, Satheesh, Ma,
  Huang, Karpathy, Khosla, Bernstein, Berg, and Fei-Fei]{russakovsky15imagenet}
Olga Russakovsky, Jia Deng, Hao Su, Jonathan Krause, Sanjeev Satheesh, Sean Ma,
  Zhiheng Huang, Andrej Karpathy, Aditya Khosla, Michael Bernstein,
  Alexander~C. Berg, and Li~Fei-Fei.
\newblock {ImageNet Large Scale Visual Recognition Challenge}.
\newblock \emph{IJCV}, 2015.

\bibitem[Salem et~al.(2019)Salem, Zhang, Humbert, Berrang, Fritz, and
  Backes]{salem2019leaks}
Ahmed Salem, Yang Zhang, Mathias Humbert, Pascal Berrang, Mario Fritz, and
  Michael Backes.
\newblock Ml-leaks: Model and data independent membership inference attacks and
  defenses on machine learning models.
\newblock In \emph{NCSS}, 2019.

\bibitem[Sankararaman et~al.(2009)Sankararaman, Obozinski, Jordan, and
  Halperin]{sankararaman2009genomicprivacy}
Sriram Sankararaman, Guillaume Obozinski, Michael~I. Jordan, and Eran Halperin.
\newblock Genomic privacy and limits of individual detection in a pool.
\newblock \emph{Nature Genetics}, 2009.

\bibitem[Shokri et~al.(2017)Shokri, Stronati, and
  Shmatikov]{shokri16membershipinference}
Reza Shokri, Marco Stronati, and Vitaly Shmatikov.
\newblock Membership inference attacks against machine learning models.
\newblock \emph{IEEE Symp. Security and Privacy}, 2017.

\bibitem[Simonyan \& Zisserman(2014)Simonyan and Zisserman]{simonyan15vgg}
Karen Simonyan and Andrew Zisserman.
\newblock Very deep convolutional networks for large-scale image recognition.
\newblock In \emph{ICLR}, 2014.

\bibitem[van~der Vaart(1998)]{vandervaart98asymptoticstatistics}
A.~W. van~der Vaart.
\newblock \emph{Asymptotic statistics}.
\newblock Cambridge Series in Statistical and Probabilistic Mathematics.
  Cambridge University Press, 1998.

\bibitem[Wang et~al.(2015)Wang, Fienberg, and Smola]{wang2015privacyforfree}
Yu-Xiang Wang, Stephen Fienberg, and Alex Smola.
\newblock Privacy for free: Posterior sampling and stochastic gradient monte
  carlo.
\newblock In \emph{ICML}, 2015.

\bibitem[Wang et~al.(2016)Wang, Lei, and Fienberg]{wang2016average}
Yu-Xiang Wang, Jing Lei, and Stephen~E Fienberg.
\newblock On-average {KL}-privacy and its equivalence to generalization for
  max-entropy mechanisms.
\newblock In \emph{PSD}. Springer, 2016.

\bibitem[Welling \& Teh(2011)Welling and Teh]{welling2011langevin}
Max Welling and Yee~Whye Teh.
\newblock Bayesian learning via stochastic gradient langevin dynamics.
\newblock In \emph{ICML}, 2011.

\bibitem[Yeom et~al.(2018)Yeom, Giacomelli, Fredrikson, and
  Jha]{yeom18privacyrisk}
Samuel Yeom, Irene Giacomelli, Matt Fredrikson, and Somesh Jha.
\newblock Privacy risk in machine learning: Analyzing the connection to
  overfitting.
\newblock In \emph{CSF}, 2018.

\bibitem[Zhang et~al.(2017)Zhang, Bengio, Hardt, Recht, and
  Vinyals]{zhang16understanding}
Chiyuan Zhang, Samy Bengio, Moritz Hardt, Benjamin Recht, and Oriol Vinyals.
\newblock Understanding deep learning requires rethinking generalization.
\newblock In \emph{ICLR}, 2017.

\end{thebibliography}
\bibliographystyle{iclr2019_conference}

\clearpage

\appendix
\begin{appendices}
\section{Derivations}

\subsection{Accuracy of the 0-1 attack}
\label{app:probaderiv}

\newcommand{\TP}{\mathrm{TP}}
\newcommand{\FP}{\mathrm{TP}}
\newcommand{\shorteq}{\textrm{=}}

We note $g_1$ the binary random variable that indicates whether $z_1$ was classified correctly, and thus considered part of the training set by the 0-1 attack.
The attack is accurate if $ g_1 = 1 $ on training images and $ g_1 = 0 $ on other images.
This happens with probability 
{\small
\begin{align}
p_{\text{bayes}} & = \prob(m_1 = g_1) \nonumber \\
&= \prob(g_1 \shorteq 1 ~|~ m_1 \shorteq 1 ) \prob(m_1 \shorteq 1) + \prob(g_1 \shorteq 0 ~|~ m_1 \shorteq 0 ) \prob(m_1 \shorteq 0) \nonumber \\
&= \lambda p_\mathrm{train} + (1-\lambda) (1 - p_\mathrm{test}).
\end{align}
}

\subsection{Gaussian data}
\label{app:gauss}

\mypar{Estimation of average distribution.} 
We assume without loss of generality that $ \mu = 0 $.
$ \theta $ is the mean of $ n $ Gaussian variables, centered on $ \mu $ with covariance $ I $.
Thus, $ \theta $ follows a Gaussian distribution, of variance $ \frac{1}{n} I $.
\begin{align}
\int_t e^{- \ell(z, t)} p(t) dt &= \frac{1}{\sqrt{ \det{\frac{2 \pi}{n} I } }} \int_t e^{\frac{- \| z - t \|^2 - n \| t\|^2}{2} } dt
\end{align}

Denoting $ \omega := \frac{z}{n+1} $, we have
\begin{align}
n \| t \|^2 + \| z - t \|^2 = (n + 1) \| t - \omega \|^2 + \frac{n}{n+1} \| z \|^2,
\end{align}
hence
\begin{align}
\int_t e^{ \frac{- \| z - t \|^2 - n \| t\|^2}{2} } dt &= \sqrt{\det{ \frac{2 \pi }{n+1} I }} e^ { - \frac{n \| z \|^2}{2(n+1)} }.
\end{align}
We have:
\begin{align}
\log \left( \int_t e^{- \ell(z, t) } p(t) dt \right) &= C - \frac{n}{2(n+1)} \| z \|^2
\end{align}

\subsection{Bound on variations of a sigmoid}
\label{sec:boundedsigma}

We show that 
\begin{align}
\sigma(u) \le \sigma(v) + |u - v|_{+}/4~~~~~\forall u, v \in \mathbb{R}. 
\label{eq:boundedsigma}
\end{align}

Since $\sigma$ is increasing, the relation is obvious for $v>u$. 

For $u>v$, we observe that 
\begin{align}
\sup_u |\sigma'(u)| = \sup_u \frac{e^{-u}}{(1 + e^{-u})^2} = \frac{1}{4}.
\end{align}
Thus, $\sigma$ is Lipschitz-continuous with constant $1/4$, which entails Equation~(\ref{eq:boundedsigma}).

\subsection{Hessian approximations}
\label{app:hessian}
We give here a rough justification of the approximation conducted in the MATT paragraph of Section \ref{sec:algorithms}.

Equation (\ref{eq:hessian_dl}) writes:
\begin{align}
\log & \left( \frac{\prob(\theta ~|~ m_1 = 1, z_1, \ZM)}{\prob(\theta ~|~ m_1 = 0, z_1, \ZM)} \right) \\
&\approx - (\theta - \theta_1^*)^T H (\theta - \theta_1^*) + (\theta - \theta_0^*)^T H (\theta - \theta_0^*).  
\end{align}
This approximation holds up to the following quantity:
\begin{align}
\delta &= \underbrace{-\frac{1}{2} \log \left( \frac{\det{H_1}}{\det{H_0}} \right) }_{\delta_1}  + \underbrace{(\theta_1^* - \theta_0^*)^T (H_1 - H_0) (\theta_1^* - \theta_0^*)}_{\delta_2}. 
\end{align}
We reason qualitatively in orders of magnitude. 
$ \theta_0^* - \theta_1^* $ has order of magnitude $ 1 / n $, and $ H_1 - H_0 $ has order of magnitude $1$, so $ \delta_2 $ has order of magnitude $ 1/n^2$.
As for $ \delta_1 $, we observe that $ H_0^{-1} (H_1 - H_0)  $ has order of magnitude $ 1/n $ and therefore 
\begin{align}
\delta_1 &= -\frac{1}{2} \log \left( \frac{\det{H_1}}{\det{H_0}} \right) \\
&=  -\frac{1}{2} \log \left( \det{I + H_0^{-1} (H_1 - H_0)  }  \right) \\
&\approx -\Tr(  H_0^{-1} (H_1 - H_0) ). 
\end{align}
Hence, $ \delta_1 $ has order of magnitude $ 1/n $ as well. 
Since the main term in Equation (\ref{eq:hessian_dl}) is in the order of $ 1/ \sqrt{n} $, $ \delta_1 $ and $ \delta_2 $ can be safely neglected.

\end{appendices}

\end{document}